\documentclass{article}
\usepackage{aistats2e}
\usepackage{amsmath, amsthm, amssymb}
\usepackage{algorithm, algorithmic}
\newtheorem{theorem}{Theorem}
\newtheorem{lemma}[theorem]{Lemma}

\theoremstyle{definition}

\newcommand{\al}{\alpha_{ij}}
\newcommand{\N}{\textsf N}
\newcommand{\h}{\theta_i}
\newcommand{\W}{W_{ij}}
\newcommand{\x}{\xi_{ij}}
\newcommand{\s}{q_i}
\newcommand{\bl}{\boldsymbol{\lambda}}
\newcommand{\ce}{\mathcal{E}}

\begin{document}
\runningauthor{Weller, Jebara}
\twocolumn[
\aistatstitle{Bethe Bounds and Approximating the Global Optimum} 
\aistatsauthor{Adrian Weller \And Tony Jebara}
\aistatsaddress{Columbia University \And Columbia Unversity}
]

\begin{abstract}
Inference in general Markov random fields (MRFs) is NP-hard, though identifying the maximum a posteriori (MAP) configuration of pairwise MRFs with submodular cost functions is efficiently solvable using graph cuts. Marginal inference, however, even for this restricted class, is in \#P. We prove new formulations of derivatives of the Bethe free energy, provide bounds on the derivatives and bracket the locations of stationary points, introducing a new technique called Bethe bound propagation. Several results apply to pairwise models whether associative or not. Applying these to discretized pseudo-marginals in the associative case we present a polynomial time approximation scheme for global optimization provided the maximum degree is $O(\log n)$, and discuss several extensions. 
\end{abstract}

\section{Introduction}

Markov random fields are fundamental tools in machine learning with broad application in areas including computer vision, speech recognition and computational biology. Two forms of inference are commonly employed: \textit{maximum a posteriori} (MAP), where the most likely configuration is returned; and \textit{marginal}, where the marginal probability distributions for each set of variables with a linking potential function are returned. In general, MAP inference is NP-hard \cite{Shi94} and marginal inference, even for pairwise models, is harder still in \#P \cite{SinJer89,Cooper90,DagumLuby93}.

An important class of MRFs, those with only unary and pairwise submodular cost functions, admits efficient MAP inference. This was first shown for binary models \cite{GrePorSeh89} and applied broadly in computer vision \cite{BoyKol04}, where the graph cuts method is particularly effective \cite{Sze06}. Recent work extended the application of this approach to multi-label submodular energies of up to third order \cite{RKAT08,SchFla06}. Yet marginal inference, even for binary pairwise models, is intractable with few known exceptions. Belief propagation is efficient (and exact) for trees, and loopy belief propagation is guaranteed to converge when the topology has one cycle \cite{YedFreWei01}.

Applying the same framework to general models, termed loopy belief propagation (LBP), has proved remarkably effective in some situations but fails in others and has no general guarantees on convergence. A key result is that belief propagation (BP) fixed points coincide with stationary points of the Bethe variational problem \cite{YedFreWei05}.  Stationary points, however, may not identify the global optimum of the the Bethe free energy. Subsequently, it was further shown that all stable BP fixed points are known to be local optima (rather than saddle points) of this problem, but not vice versa \cite{Hes03,Hes06}. Variational methods demonstrate that minimizing the Bethe free energy should deliver a good approximation to the true marginal distribution and recently \cite{Ruo12} proved that for submodular MRFs, the Bethe optimum is an upper bound on the true free energy and thus yields a desirable lower bound on the partition function. 


Marginal inference is a crucial problem in probabilistic systems. A noteworthy example is the Quick Medical Reference (QMR) problem \cite{Shwe91}, a graphical model involving 600 diseases and 4000 possible findings. Therein, medical diagnostics are performed by computing the posterior marginal probability of each disease given a set of possible findings. The marginal distribution over the presence of a disease must often be precisely estimated in order to determine the course of medical treatment. Thus, we seek the probability that a patient suffers from a condition, rather than the MAP estimate, which could be very different. 

Marginal inference also arises during learning or parameter estimation in Markov random fields. For instance, computing the gradients of a partition function in a maximum likelihood estimation procedure is equivalent to marginal inference. In learning problems, the intractability of the marginal inference problem requires the exploration of marginal approximation schemes \cite{Gan08}. However, in the general case, {\em both} exact marginal inference and approximate marginal inference are NP-hard \cite{Cooper90,DagumLuby93}. 


\subsection{Contribution}
We derive various properties of the Bethe free energy and apply them to discretized pseudo-marginals to prove a polynomial-time approximation scheme (PTAS) for the global minimum of the Bethe free energy for binary pairwise associative MRFs.

The idea is that if we can find the optimal discretized point on a sufficiently fine mesh that covers all possible locations of an optimum point within a distance of $\delta$, then we can bound the difference to the optimum by $\frac{1}{2} \Lambda \delta^2$ where $\Lambda$ is the greatest directional second derivative. To our knowledge, we present the first rigorous bounds on $\Lambda$. One reason this is difficult is that derivatives tend to infinity as singleton marginals approach the boundary cases of $0$ or $1$. Hence we need to prove bounds on the location away from these edges. 

We first prove various bounds including on the location of any stationary point of the Bethe free energy, as well as on the true marginals. In doing this we develop Bethe bound propagation (BBP) which sometimes produces remarkably tight bounds by itself. We then consider the second derivatives with a view to bounding $\Lambda$. Additional analysis allows us to prove that the discretized multi-label problem is submodular on any mesh and hence the discretized optimum can be found efficiently using graph cuts \cite{SchFla06}.

Various extensions are discussed in the closing section, including applications to non-associative models, to models that are themselves multi-label and to models with higher order terms.

\subsection{Related work}

A variety of heuristics have been proposed for marginal inference problems. Marginal inference in the QMR medical diagnostic problem has been explored with Markov Chain Monte Carlo (MCMC) \cite{MacKay98,ShweCooper91,DagumHorvitz93} methods, variational methods \cite{JaakkolaJordan99}, and search methods \cite{Dechter97}. Many of these heuristics are restricted to certain classes of graphical model (such as QMR). Here we explore another approach to approximate marginal inference by minimizing the Bethe free energy.

The minimization of Bethe free energy is often approached using loopy Belief propagation. However, there are few guarantees on the rate of convergence of LBP which prevent it from functioning as a PTAS for Bethe minimization \cite{WaiJor08}. An important contribution \cite{WelTeh01} showed that the Bethe free energy of a binary pairwise MRF may be considered as a function only of the singleton marginals, however this connection was provided without convergence results. 

A PTAS was recently proposed \cite{Shin12} for the location of a point whose derivative of the Bethe free energy has magnitude less than $\epsilon$. However, this identifies only an approximately stationary point (which may not be even a local minimum) that could be arbitrarily far from the global optimum. That result applies for a general binary pairwise MRF subject to an edge sparsity restriction that the maximum degree is $O(\log n)$. Here we primarily focus on associative models with the same degree restriction, but our deliverable not only satisfies the property in \cite{Shin12} but importantly is also guaranteed to have Bethe free energy within $\epsilon$ of the optimum. 

We note that the PTAS in \cite{Shin12} may provide the global optimum when the fixed point is unique and recent work \cite{Wat11} has enumerated necessary and sufficient conditions for uniqueness. Nevertheless, aside from these restricted settings, there are no prior polynomial-time methods for finding or rigorously approximating the global minimum of the Bethe free energy. Earlier work considered discretizations of pseudo-marginals but presented incomplete results \cite{Kor12}. We go significantly further in deriving additional key results which together admit the PTAS. These include explicit forms and bounds on the second derivatives, on the third derivatives and on the locations of stationary points.

\section{Preliminaries \& Notation}

We focus on a binary pairwise MRF over $n$ variables $X_1,\dots,X_n$ with topology $(\mathcal{V},\mathcal{E})$ and generally follow the notation of \cite{WelTeh01}. We assume\footnote{The energy $E$ can always be thus reparameterized with finite $\h$ and $W_{ij}$ terms provided $p(x)>0 \; \forall x$. There are reasonable distributions where this does not hold, i.e. $\exists x: p(x)=0$ but this can often be handled by assigning such configurations a sufficiently small positive probability $\epsilon$.\label{fn:finite}}
\begin{equation}\label{eq:E}
p(x)=\frac{e^{-E(x)}}{Z}, \; E=-\sum_{i \in \mathcal{V}} \h x_i -\sum_{(i,j)\in \mathcal{E}} \W x_i x_j
\end{equation}
where the partition function $Z=\sum_x e^{-E(x)}$ is a normalizing constant. Let $F$ be the Bethe free energy, so $F=E-S$ where $S$ is the Bethe approximation to the true entropy, $S=\sum_{(i,j)\in \ce} S_{ij} + \sum_{i \in \cal{V}} (1-z_i) S_i$. $S_{ij}$ is the entropy of a pseudo-marginal of $(X_i,X_j)$ on the local polytope, $S_i$ is the entropy of the singleton distribution and $z_i$ is the degree of $i$, that is the number of variables to which $X_i$ is adjacent. We assume the model is connected so all $z_i \geq 1$. For each node $i$ define sum of positive and negative incident edge weights: $W_i=\sum_{j \in \N(i): W_{ij}>0} W_{ij}$, $V_i=-\sum_{j \in \N(i): W_{ij}<0} W_{ij}$ where $\N(i)$ indicates the neighbors of node $i$. For a pseudo-marginal distribution $q$, let $q_i=p(X_i=1)$. Consistency and normalization constraints from the local polytope imply
\begin{equation}\label{eq:mu}
\mu_{ij}=\begin{pmatrix} 1 + \x -\s -q_j & q_j-\x \\ \s - \x & \x \end{pmatrix}
\end{equation}
for some $\x \in [0,\min(\s, q_j)]$, where $\mu_{ij}(a,b)=p(X_i=a,X_j=b)$ is the pairwise marginal. Let $\al=e^{\W}-1$. $\al=0 \Leftrightarrow \W=0$ may be assumed not to occur else the edge $(i,j)$ may be deleted. $\al$ has the same sign as $\W$, if positive then the edge $(i,j)$ is associative; if negative then the edge is repulsive. The MRF is associative if all edges are associative. As in \cite{WelTeh01}, one can solve for $\x$ explicitly in terms of $\s$ and $q_j$ by minimizing the free energy, leading to a quadratic equation with real roots
\begin{equation}\label{eq:xi2}
\alpha_{ij}\xi_{ij}^2 - [1+\alpha_{ij}(q_i+q_j)]\xi_{ij}+(1+\alpha_{ij})q_i q_j=0.
\end{equation}
For $\al>0$, $\x(q_i,q_j)$ is the lower root, for $\al<0$ it is the higher. Notice that when $\al=0$ (no edge relationship) this reduces as expected to $\x=p(X_i=1,X_j=1)=p(X_i=1)p(X_j=1)=q_i q_j$.

$S_{ij}$ is the entropy of $\mu_{ij}(q_i,q_j)$. Hence
\begin{equation}\label{eq:F}
\begin{split}
F(q)=&\sum_{(i,j)\in \ce} -\big(\W \x +S_{ij}(q_i,q_j) \big) \\
&\quad +\sum_{i \in \cal{V}} \big( -\h q_i + (z_i-1) S_i(q_i) \big).
\end{split}
\end{equation}
Collecting the pairwise terms for one edge, define
\begin{equation}\label{eq:f}
f_{ij}(q_i,q_j)=-\W \x(q_i, q_j) -S_{ij}(q_i,q_j).
\end{equation}
We are interested in \textit{discretized pseudo-marginals} where for each $q_i$ we restrict its possible values to a discrete set $D_i$ of points in $[0,1]$. Note we may often have $D_i \ne D_j$. Let $\mathcal{D}=\prod_{i \in V} D_i$. 

In \cite{WelTeh01}, the first partial derivative of the Bethe free energy is derived as
\begin{equation}\label{eq:1stderiv}
\frac{\partial F}{\partial q_i} = -\theta_i + \log Q_i \: \text{, where}
\end{equation} 
$$Q_i = \frac{ (1-q_i)^{z_i-1} }  { q_i^{z_i-1} } 
\frac{\prod_{j \in \N(i)} (q_i-\xi_{ij})}{\prod_{j \in \N(i)} (1+\xi_{ij}-q_i-q_j)}.$$

Recall the sigmoid function $\sigma(x)=1/(1+\exp(-x))$ which will be used for Bethe bounds. We write $A_i$ for the lower bound of $q_i$ and $B_i$ for the lower bound of $1-q_i$ so $A_i \leq q_i \leq (1-B_i)$. Define $\eta_i=\min(A_i,B_i)$.


\subsection{Submodularity}
In our context, a pairwise multi-label function on a set of ordered labels $X_{ij}=\{1,\dots,K_i\} \times \{1,\dots,K_j\}$ is \textit{submodular} if 
\begin{equation}\label{eq:sub}
\forall x, y \in X_{ij}, \; f(x \wedge y) + f(x \vee y) \leq f(x) + f(y)
\end{equation} where for $x=(x_1,x_2)$ and $y=(y_1,y_2)$, $(x \wedge y)=(\min(x_1,y_1),\min(x_2,y_2))$ and $(x \vee y)=(\max(x_1,y_1),\max(x_2,y_2))$. For binary variables this is equivalent to associativity. 

The key property for us is that if all the pairwise cost functions $f_{ij}$ over $D_i \times D_j$ from \eqref{eq:f} are submodular 
then the global discretized optimum may be found efficiently as a multi-label MAP inference problem using graph cuts \cite{SchFla06}. 

\section{Bounds \& Bethe bound propagation}

We use the technique of flipping variables, i.e. considering $Y_i=1-X_i$. Flipping a variable flips the parity of all its incident edges so associative $\leftrightarrow$ repulsive. Flipping both ends of an edge leaves its parity unchanged.

\subsection{Flipping all variables}\label{sec:flipall}

Consider a new model with variables $\{Y_i=1-X_i, i=1,\dots,n\}$ and the same edges. Instead of $\theta_i$s and $W_{ij}$s, let the new model have parameters $\phi_i$ and $V_{ij}$. We identify values such that the energies of all states are maintained up to a constant.\footnote{Any constant difference will be absorbed into the partition function and leave probabilities unchanged.}
\begin{align*}
E &= -\sum_{i \in \cal{V}} \theta_i X_i - \sum_{(i,j)\in \ce} W_{ij} X_i X_j \\
&= const -\sum_{i \in \cal{V}} \phi_i (1-X_i) - \sum_{(i,j) \in \ce} V_{ij}(1-X_i)(1-X_j).
\end{align*}
Matching coefficients yields
\begin{equation}\label{eq:flipall}
V_{ij}=W_{ij}, \; \phi_i=-\theta_i-\sum_{j \in \N(i)} W_{ij}=-\theta_i-W_i.
\end{equation}
If the original model was associative, so too is the new.

\subsection{Flipping some variables}\label{sec:flipsome}

Sometimes we flip only a subset $\cal{R} \subseteq \cal{V}$ of the variables. This can be useful, for example, to make the model locally associative around a variable, which can always be achieved by flipping just those neighbors to which it has a repulsive edge. Let $Y_i = 1-X_i$ if $i \in \cal{R},$ else $Y_i=X_i$ for $i \in \cal{S}$, where $\cal{S}=\cal{V} \setminus \cal{R}$. Let $\mathcal{E}_t=\{$edges with exactly $t$ ends in $\cal{R}\}$ for $t=0,1,2$. 


As in \ref{sec:flipall}, solving for $V_{ij}$ and $\phi_i$ such that energies are unchanged up to a constant,
\begin{align}
\label{eq:flipsome}
V_{ij} &= \begin{cases} W_{ij} & \mspace{-1mu}(i,j) \in \mathcal{E}_0 \cup \mathcal{E}_2,\\
-W_{ij} & \mspace{-1mu} (i,j) \in \mathcal{E}_1
\end{cases} \nonumber\\
\mspace{-4mu}\phi_i &= \begin{cases} \h + \sum_{(i,j) \in \mathcal{E}_1} W_{ij} & \mspace{-1mu} i \in \cal{S}, \\ -\h - \sum_{(i,j) \in 
\mathcal{E}_2} W_{ij} & \mspace{-4mu} i \in \cal{R}. \end{cases}
\end{align}

\begin{lemma}\label{lem:flipBethe}
Flipping any set of variables changes affected pseudo-marginal matrix entries' locations but not values. The Bethe free energy is unchanged up to a constant, hence the locations of stationary points are unaffected.
\end{lemma}
\begin{proof}
By construction energies are the same up to a constant. The singleton entropies are symmetric functions of $q_i$ and $1-q_i$ so are unaffected. The impact on pseudo-marginal matrix entries follows directly from definitions. Thus Bethe entropy is unaffected.
\end{proof}

\subsection{Bounds}

We derive several results that are useful in bounding the Bethe free energy as well as the marginals.
\begin{lemma}\label{lem:newst}
$\al \geq 0 \Rightarrow \x \geq q_i q_j, \al \leq 0 \Rightarrow \x \leq q_i q_j$
\end{lemma}
\begin{proof}
The quadratic equation \eqref{eq:xi2} for $\x$ may be rewritten  $\x-\s q_j = \al (\s-\x)(q_j-\x)$. Both terms in parentheses on the right are elements of the pseudo-marginal matrix $\mu$ so are constrained to be $\geq 0$.
\end{proof}
This simple result is sufficient to bound the location of stationary points of the Bethe free energy away from the edges of $0$ and $1$, though we improve the bounds in Lemma \ref{lem:xiblb}.
\begin{theorem}\label{thm:qsand1}
If all edges incident to $X_i$ are associative then at any stationary point of the Bethe free energy, $\sigma(\theta_i) \leq q_i \leq \sigma(\theta_i+W_i)$. Remark exactly the same sandwich result holds for the true marginal $p_i$.
\end{theorem}
\begin{proof}
We first prove the left inequality. Consider \eqref{eq:1stderiv}. Using $\al>0 \; \forall j \in \N(i)$ and Lemma \ref{lem:newst} we have
\begin{align*}
Q_i &= \frac{ \prod_{j \in \N(i)} (q_i-\xi_{ij}) }  { q_i^{z_i-1} } 
\frac{ (1-q_i)^{z_i-1} }{\prod_{j \in \N(i)} (1+\xi_{ij}-q_i-q_j)} \\
&\leq \frac{ \prod_{j \in \N(i)} q_i(1-q_j) }  { q_i^{z_i-1} } 
\frac{(1-q_i)^{z_i-1}}{\prod_{j \in \N(i)} (1-q_i)(1-q_j)} \\
&= \frac{q_i}{1-q_i} \text{ which gives the result.}
\end{align*}
To obtain the right inequality, flip all variables as in section \ref{sec:flipall}. Using the first inequality, \eqref{eq:flipall} and Lemma \ref{lem:flipBethe} yields $1-q_i \geq \sigma(-\h-W_i) \Leftrightarrow q_i \leq \sigma(\h+W_i)$ since $1-\sigma(-x)=\sigma(x)$. To show the result for the true marginal, let $m_{i=a}=\sum_{x:x_i=a} \exp(\sum_{i \in V} \h x_i +\sum_{(i,j)\in E} \W x_i x_j)$
then using \eqref{eq:E}, $p_i = \frac{m_{i=1}}{m_{i=1}+m_{i=0}}$. Since all $W_{ij}>0$ the result follows.
\end{proof}

Using \eqref{eq:flipsome} we obtain a more powerful corollary.
\begin{theorem}\label{thm:qsand2}
For general edge types (associative or repulsive), let $W_i=\sum_{j \in \N(i): W_{ij}>0} W_{ij}$, $V_i=-\sum_{j \in \N(i): W_{ij}<0} W_{ij}$. At any stationary point of the Bethe free energy, $\sigma(\theta_i-V_i) \leq q_i \leq \sigma(\theta_i+W_i)$. The same sandwich result holds for the true marginal $p_i$.
\end{theorem}
\begin{proof}
Using \eqref{eq:flipsome}, flip all variables adjacent to $X_i$ with a repulsive edge, i.e. set $\mathcal{R}=\{j \in \N(i):W_{ij}<0\}$. The resulting new model is fully associative around $X_i$ so we may apply Theorem \ref{thm:qsand1} to yield the result.
\end{proof}
The following lemma will be useful.
\begin{lemma}\label{lem:st}
For $q_i,q_j\in[0,1], 0 \leq \s+q_j-2\s q_j \leq 1$.
\end{lemma}
\begin{proof}
Let $f=\s+q_j-2\s q_j$. To show the left inequality, consider $m=\min(q_i,q_j)$ and $M=\max(q_i,q_j)$, then $f \geq 2m(1-M) \geq 0$. For the right inequality observe $1-f=(1-\s)(1-q_j)+\s q_j \geq 0.$
\end{proof}

\begin{lemma}[Better lower bound for $\x$]\label{lem:xiblb}
If $\al>0$, then $\xi_{ij}\geq q_iq_j+\al q_i q_j (1-q_i)(1-q_j) / [1+\alpha_{ij} (q_i+q_j-2q_iq_j)]$, equality only possible at an edge, i.e. one or both of $q_i, q_j \in \{0,1\}$.
\end{lemma}
\begin{proof}
Write $\xi_{ij}=q_iq_j+y$ and substitute into \eqref{eq:xi2},
$$\al y^2 -y[1+\al(\s+q_j-2\s q_j)]+\al \s q_j(1-\s)(1-q_j)=0.$$
We have a convex parabola which at $y=0$ is above the abscissa (unless $q_i$ or $q_j \in \{0,1\}$) and has negative gradient by Lemma \ref{lem:st}. Hence all roots are at $y \geq 0$ and given convexity we can bound below using the tangent at $y=0$ which yields the result.
\end{proof}

\begin{lemma}[Upper bound for $\x$]\label{lem:xiub}
If $\al>0$, then $q_j-\xi_{ij} \geq \frac{q_j(1-q_i)}{1+\al(\s+q_j-2\s q_j)} \geq \frac{q_j(1-q_i)}{1+\al} 
\\q_i-\xi_{ij} \geq \frac{q_i(1-q_j)}{1+\al(\s+q_j-2\s q_j)} \geq \frac{q_i(1-q_j)}{1+\al} 
$. \\
Also $\x \leq m(\al + M) / (1+\al) \Rightarrow \x-\s q_j \leq \frac{\al m (1-M)}{1+\al}$.
\end{lemma}
\begin{proof}
We prove the first inequality. The second follows by Lemma \ref{lem:st} and those for $q_i-\x$ follow by symmetry. The final inequality follows by combining the earlier ones. Let $\x=q_j+y$ and substitute into \eqref{eq:xi2}
$$ \al y^2 + y[\al (q_j-\s) - 1] + q_j(\s-1)=0.$$
The function is a convex parabola which at $y=0$ is at $q_j(\s-1) \leq 0$.\footnote{This confirms neatly that we must take the left root else $y>0 \Rightarrow \mu_{01}<0$ (a contradiction).} From Lemma \ref{lem:newst} we know that the left root is at $\x \geq \s q_j$ so we may take the derivative there, i.e. at $q_j+y=\s q_j \Leftrightarrow y=q_j(\s-1)$ and by convexity use this to establish a lower bound for $q_j-\x$. That derivative is $2\al \s q_j - 2\al q_j + \al q_j - \al \s -1 = -[1+\al (q_i+q_j-2q_i q_j)]$. 
\end{proof}

\begin{lemma}\label{lem:mugz}
Unless $q_i$ or $q_j \in \{0,1\}$, all entries of the pseudo-marginal $\mu_{ij}$ are strictly $> 0$, whether $(i,j)$ is associative or repulsive.\footnote{Here we assume $\al$ is finite, see footnote \ref{fn:finite}.}
\end{lemma}
\begin{proof}
First assume $\al>0$. Considering \eqref{eq:mu} and using Lemmas \ref{lem:newst} and \ref{lem:xiub}, we have that element-wise 
\begin{equation}\label{mulb}
\mu_{ij} \geq \begin{pmatrix} (1-\s)(1-q_j) & q_j(1-\s)/(1+\al) \\ q_i(1-q_j)/(1+\al) & \s q_j \end{pmatrix}
\end{equation}
which proves the result for this case. If $\al<0$ then flip either $q_i$ or $q_j$. As in the proof of Lemma \ref{lem:flipBethe}, pseudo-marginal entries change position but not value.
\end{proof}

\subsection{Bethe bound propagation (BBP)}

We have already derived bounds on stationary points in Theorems \ref{thm:qsand1} and \ref{thm:qsand2}. Here we show for variables with only associative edges how we can iteratively improve these bounds, sometimes with striking results. Note that a fully associative model is not required, and as in section \ref{sec:flipsome}, \textit{any} model may be selectively flipped to yield local associativity around a particular node.

We first assume all $\al \geq 0$ and adopt the approach of Theorem \ref{thm:qsand1}, now using the better bound from Lemma \ref{lem:xiblb} to obtain
\begin{align*}
q_i-\xi_{ij} &\leq  q_i-q_iq_j-\frac{\alpha_{ij} q_i q_j (1-q_i)(1-q_j)}{ 1+\alpha_{ij}(q_i+q_j-2q_iq_j)}\\
&=q_i(1-q_j)\Big[1- \frac{\al q_j(1-q_i)}{1+\alpha_{ij}(q_i+q_j-2q_iq_j)}\Big], 
\end{align*}
\begin{align*}
&1+\x -q_i-q_j \geq \\
&1+\s q_j -\s -q_j +\frac{\alpha_{ij} q_i q_j (1-q_i)(1-q_j)}{ 1+\alpha_{ij}(q_i+q_j-2q_iq_j)} \\
&=(1-\s)(1-q_j)\Big[1+ \frac{\al \s q_j}{1+\alpha_{ij}(q_i+q_j-2q_iq_j)}\Big].
\end{align*}

Hence $Q \leq \frac{\s}{1-\s} \prod_{j \in \N(i)} R_{ij}^{-1}$ where 
\begin{align*}
R_{ij} &= \frac{1+\frac{\al \s q_j}{1+\al(\s+q_j-2\s q_j)}} {1-\frac{\al q_j(1-\s)}{1+\al(\s+q_j-2\s q_j)}} = 1 + \frac{\al q_j}{1+ \al \s (1-q_j)},
\end{align*}
monotonically increasing with $q_j$ and decreasing with $q_i$. Hence 
\begin{equation}\label{eq:Rij}
e^{W_{ij}} = 1 + \al \geq R_{ij} \geq L_{ij}:=1+ \frac{\al A_j}{1+\al (1-B_i) (1-A_j)}
\end{equation}
Using Theorem \ref{thm:qsand1} we initialize $A_i=\sigma(\theta_i)$ and $B_i=1-\sigma(\theta_i+W_i)$.

Using \eqref{eq:1stderiv}, at any stationary point we must have $$q_i \geq 1 / [1+\exp(-\h)/L_i]$$ where $L_i = \prod_{j \in \N(i)} L_{ij}$. Intuitively, in an associative model, if variable $i$ has neighbors $j$ which are likely to be $1$ (i.e. high $A_j$) then this pulls up the probability that $i$ will be 1 (i.e. raises $A_i$).

Flipping all variables, 
$$1-q_i \geq 1/[1+\exp(\h+W_i)/U_i]$$ where $U_i=\prod_{j \in \N(i)} U_{ij}$ with $$e^{-W_{ij}} \geq U_{ij}:=1+ \frac{\al B_j}{1+\al (1-A_i) (1-B_j)}.$$ 
It is also possible to write this as $$ \sigma(\h + \log L_i) \leq \s \leq \sigma(\h + W_i -\log U_i). $$

This establishes a message passing type of algorithm for iteratively improving the bounds $\{A_i,B_i\}$. Repeat until convergence: 
\begin{align*}
\text{  new } A_i &\leftarrow \left(1+\exp(-\h)/L_i \right)^{-1} \\
\text{  new } B_i &\leftarrow \left (1+\exp(\h+W_i)/U_i \right)^{-1}\\
\text{  recompute } &L_i, U_i \text{ using new } A_i, B_i.
\end{align*}

\begin{lemma} \label{lem:mono}
At every iteration, all of $A_i, B_i, L_{ij}, U_{ij}$ monotonically increase.
\end{lemma}
\begin{proof}
All of the dependencies are monotonically increasing on all inputs. The first iteration yields an increase since each $L_{ij}, U_{ij} > 1$.
\end{proof}

Since $A_i+B_i\leq 1$, each is bounded above and we achieve monotonic convergence. Combining this with the main global optimization approach can dramatically reduce the range of values that need be considered, leading to significant time savings. Convergence is rapid even for large, densely connected graphs. Each iteration takes $O(|\cal{E}|)$ time; a good heuristic is to run for up to 20 iterations, terminating early if all parameters improve by less than a threshold value. This adds negligible time to the global optimization.


This procedure alone can produce impressive results. For example, running on a $100$-node graph with independent random edge probability $0.04$ (hence average degree $4$), each $W_{ij}$ and $\h$ drawn randomly from Uniform $[0,1]$ and then adjusting $\h \leftarrow \h - \sum_{j \in \N(i)} W_{ij}/2$ in order to be unbiased, convergence takes about 11 iterations yielding final average bracket width of $0.05$ after starting with average bracket width of $0.40$. Greater connectivity, higher edge strengths and smaller individual node potentials make the problem more challenging and may widen the returned final brackets significantly.

\subsection{BBP for general models}

A repulsive edge $(i,j)$ may always be flipped to associative by flipping variable $j$, which flips its Bethe bounds $A_j \leftrightarrow B_j$. Using Theorem \ref{thm:qsand2} we can extend the analysis above to run BBP on any model, see Algorithm 1
. Performance in terms of convergence speed and final bracket width is similar for associative and non-associative models.
\begin{algorithm}
\caption{BBP for a general binary pairwise model}
\label{alg:gbbp}
\begin{algorithmic}
	\STATE \COMMENT{Initialize}
	\FORALL {$i \in \mathcal{V}$}
		\STATE $W_i=\sum_{j \in \N(i): W_{ij}>0} W_{ij}$,
	\STATE $V_i=-\sum_{j \in \N(i): W_{ij}<0} W_{ij}$,
	\STATE $A_i=\sigma(\theta_i-V_i)$, $B_i=1-\sigma(\theta_i+W_i)$ 
	\ENDFOR
	\FORALL {$(i,j) \in \mathcal{E}$}
		\STATE $\alpha_{ij} =\exp (|W_{ij}|) -1$
	\ENDFOR	
	\REPEAT
		\FORALL {$i \in \mathcal{V}$}
			\STATE $L_i=1$, $U_i=1$ \COMMENT{Initialize for this pass}
			\FORALL {$j \in \N(i)$}
				\IF {$W_{ij}>0$} 
					\STATE \COMMENT{Associative edge}
					\STATE $L_i *=1+ \frac{\al A_j} {1+\al (1-B_i) (1-A_j)}$
					\STATE $U_i *=1+ \frac{\al B_j} {1+\al (1-A_i) (1-B_j)}$
				\ELSE 
					\STATE \COMMENT{Repulsive edge}
					\STATE $L_i *=1+ \frac{\al B_j} {1+\al (1-B_i) (1-B_j)}$
					\STATE $U_i *=1+ \frac{\al A_j} {1+\al (1-A_i) (1-A_j)}$
				\ENDIF
			\ENDFOR
			\STATE $A_i = 1 / (1+\exp(-\h + V_i) / L_i)$
			\STATE $B_i = 1 / (1+\exp(\h + W_i) / U_i)$
		\ENDFOR	
	\UNTIL {All $A_i$,$B_i$ changed by $<$ THRESH \textbf{or} run MAXITER times}
	\STATE \COMMENT{Suggested THRESH$=0.002$, MAXITER$=20$}
\end{algorithmic}
\end{algorithm}
	
%
%
%

\section{Higher derivatives \& submodularity}

We first derive a novel result for the second derivatives of an edge which will be crucial later for bounding the error of the discretized global optimum and also will allow us to show that  the discretized multi-label problem is submodular.

\subsection{Second derivatives for each edge}

\begin{theorem}\label{thm:2nderiv}
For any edge $(i,j)$, for any $\al$, writing $f=f_{ij}$ and $\mu_{ab}=\mu_{ij}(a,b)$ from \eqref{eq:mu},
\begin{align*}
\frac{\partial^2f}{\partial q_i^2} &= \frac{1}{T_{ij}} q_j(1-q_j) \\
\frac{\partial^2f}{\partial q_i \partial q_j} = \frac{\partial^2f}{\partial q_j \partial q_i} &= \frac{1}{T_{ij}} (\mu_{01}\mu_{10}-\mu_{00}\mu_{11}) \\
\frac{\partial^2f}{\partial q_j^2} &= \frac{1}{T_{ij}} q_i(1-q_i)
\end{align*} 
where $T_{ij}= \s q_j(1-\s)(1-q_j) -(\x-\s q_j)^2 \geq 0$ with equality only for $q_i$ or $q_j \in \{0,1\}$. Further $\mu_{01}\mu_{10}-\mu_{00}\mu_{11} = \s q_j - \x$ and has the sign of $-\al$.
\end{theorem}
\begin{proof}
We begin with the same approach as \cite{Kor12} but extend the analysis and derive stronger results.

For notational convenience add a third pseudo-dimension restricted to the value $1$. Let $\textbf{y}=(y_1,y_2,y_3)$ be the vector with components $y_1=x_i$, $y_2=x_j$ and $y_3=1$ where $x_i, x_j \in \mathbb{B}$. Define $\pi(\textbf{y})=\mu_{ij}(x_i,x_j)$, and $\phi(\textbf{y})=\W$ if $\textbf{y}=(1,1,1)$ or $\phi(\textbf{y})=0$ otherwise. Let $\textbf{r}=(q_i,q_j,1)$. Define function $h$ used in entropy calculations as $h(z)=-z \log z$. 
 
Consider \eqref{eq:f} but instead of solving for $\x$ explicitly, express $f$ as an optimization problem, minimizing free energy subject to local consistency and normalization constraints in order to use techniques from convex optimization. We have $f(q_i,q_j)=g(\textbf{r})$ where 
\begin{align}\label{eq:g}
g(\textbf{r}) & = & \min_\pi 
\sum_{\bf y} \big( -\phi(\textbf{y})\pi(\textbf{y})-h(\pi(\textbf{y})) \big) \nonumber \\
&& {\rm s.t.} \; \sum_{\textbf{y}:y_k=1} \pi(\textbf{y})=r_k \:\: k=1,2,3.
\end{align}

The Lagrangian can be written as 
\begin{align*}
L_\textbf{r}(\pi, \boldsymbol{\lambda} )=\sum_\textbf{y} [ (-\phi(\textbf{y})-\langle \textbf{y},\bl \rangle )\pi(\textbf{y}) - h(\pi(\textbf{y}))  
]
+ \langle \textbf{r},\boldsymbol{\lambda} \rangle
\end{align*}
and its derivative is
\begin{align*}
\frac{\partial L_\textbf{r}(\pi, \boldsymbol{\lambda} ) } {\partial \pi} =  -\phi(\textbf{y})-\langle \textbf{y},\bl \rangle +1 + \log \pi
\end{align*}
which yields a minimum at 
\begin{equation}\label{eq:pi}
\pi_{\boldsymbol{\lambda}} (\textbf{y}) = \exp(\phi(\textbf{y}) + \langle \textbf{y},\bl \rangle -1).
\end{equation}
Since the minimization problem in (14) is convex and satisfies the weak Slater's condition (the constraints are affine), strong duality applies and $g(\textbf{r})=\max_{\bl} G(\textbf{r},\bl) = G(\textbf{r},\bl^*(\textbf{r}))$ where the dual is simply
\begin{equation}\label{eq:G}
G(\textbf{r},\bl)=\min_\pi L_\textbf{r}(\pi, \boldsymbol{\lambda} ) = -\sum_{\textbf{y}} \pi_{\boldsymbol{\lambda}} (\textbf{y})  + \langle \textbf{r},\bl \rangle.
\end{equation}

Let $D_k(\textbf{r},\bl) = \frac{\partial G(\textbf{r},\bl)}{\partial \lambda_k}$ then $D_k(\textbf{r},\bl^*)=0, \; k=1,2,3$.

Hence $\frac{\partial g}{\partial r_k} = \frac{\partial G}{\partial r_k} = \lambda_k$ using \eqref{eq:G}. Focusing on our goal of obtaining second derivatives, we consider $\frac{\partial^2 g}{\partial r_l \partial r_k} = \frac{\partial \lambda_k}{\partial r_l}$ which we shall express in terms of $C_{kl}:=\frac{\partial^2 G}{\partial \lambda_l \partial \lambda_k}=\frac{\partial D_k}{\partial \lambda_l}$.

Differentiating $D_k(\textbf{r},\bl^*)=0$ with respect to $r_l$,
\begin{equation*}
0=\frac{\partial D_k(\textbf{r},\bl^*)} {\partial r_l} = \frac{\partial D_k}{\partial r_l} + 
\sum_{p=1}^3 \frac{\partial D_k}{\partial \lambda_p} \frac{\partial \lambda_p}{\partial r_l} \quad k,l=1,2,3.
\end{equation*}
Considering \eqref{eq:G}, $\frac{\partial D_k}{\partial r_l} = \frac{\partial^2 G}{\partial r_l \partial \lambda_k}=\delta_{kl}$ hence $0=\delta_{kl} + \sum_p C_{kp} \frac{\partial^2 g}{\partial r_l \partial r_p} $. Thus $\frac{\partial^2 g}{\partial r_l \partial r_k} = -[C^{-1}]_{kl}$. Using its definition and \eqref{eq:G}, we have
\begin{align*}
C_{kl} &= \frac{\partial^2 G}{\partial \lambda_l \partial \lambda_k}
= \frac{\partial}{\partial \lambda_l} \Big(-\sum_{\textbf{y}} y_k  \pi_{\boldsymbol{\lambda}} (\textbf{y})  +r_k \Big) \\
&= -\sum_{\textbf{y}} y_k y_l \pi_{\boldsymbol{\lambda}} (\textbf{y}) 
= -\sum_{\textbf{y}:y_k=y_l=1} \pi_{\boldsymbol{\lambda}} (\textbf{y}).
\end{align*}
Earlier work \cite{Kor12} stopped here, recognizing that $\det C \leq 0$. We more precisely characterize this matrix \\
\begin{equation}\label{eq:C}
C=-\begin{pmatrix} \mu_{10}+\mu_{11} & \mu_{11} & \mu_{10}+\mu_{11} \\
\mu_{11} & \mu_{01}+\mu_{11} & \mu_{01}+\mu_{11} \\
\mu_{10}+\mu_{11} & \mu_{01}+\mu_{11} & 1 \end{pmatrix}\end{equation}
Recall constraints $\mu_{00}+\mu_{01}+\mu_{10}+\mu_{11}=1$, $\mu_{01}+\mu_{11}=q_j$, $\mu_{10}+\mu_{11}=q_i$. Note $C$ is symmetric.

Applying our result above and using Cramer's rule,
$$ \frac{\partial^2f}{\partial q_i^2} = \frac{\partial^2g}{\partial r_1^2} = - \frac{1}{\det C} (\mu_{01}+\mu_{11})(\mu_{00}+\mu_{10}) =  \frac{q_j (1-q_j)}{-\det C}$$
$$\frac{\partial^2f}{\partial q_i \partial q_j} = \frac{\partial^2f}{\partial q_j \partial q_i} = \frac{\partial^2g}{\partial r_1 \partial r_2} = \frac{ (\mu_{01}\mu_{10}-\mu_{00}\mu_{11}) }{-\det C}$$
$$ \frac{\partial^2f}{\partial q_j^2} = \frac{\partial^2g}{\partial r_2^2} = - \frac{1}{\det C}  (\mu_{10}+\mu_{11})(\mu_{00}+\mu_{01}) =  \frac{q_i (1-q_i)}{-\det C}.$$

Using \eqref{eq:C} and simplifying, we obtain $-\det C= \mu_{00}\mu_{10}\mu_{11}+\mu_{10}\mu_{11}\mu_{01}+\mu_{11}\mu_{10}\mu_{00}+\mu_{01}\mu_{00}\mu_{10}$. 
By Lemma \ref{lem:mugz} this is strictly $>0$ unless $q_i$ or $q_j \in \{0,1\}$. Substituting in terms from \eqref{eq:mu} and simplifying establishes $-\det C = T_{ij}$ from the statement of the theorem, and $\mu_{01}\mu_{10}-\mu_{00}\mu_{11} = \s q_j - \x$. The sign follows from Lemma \ref{lem:newst} or observing from \eqref{eq:pi} that $\frac{\mu_{00}\mu_{11}}{\mu_{01}\mu_{10}} = e^{W_{ij}} = \al +1$.
\end{proof}
Note that stronger edge interactions lead through higher $|\al|$ to greater $(\x - \s q_j)^2$ and hence larger second derivatives. 

\subsection{Third derivatives for each edge}

\begin{lemma}[Finite 3rd derivatives]\label{lem:bound3}
For any edge $(i,j)$ with $\al>0$, if $q_i,q_j \in (0,1)$ then all third derivatives exist and are finite.
\end{lemma}
\begin{proof}
Using Theorem \ref{thm:2nderiv} noting $T_{ij}>0$ strictly and considering \eqref{eq:mu}, it is sufficient to show $\frac{\partial \x}{\partial q_k}$ is finite. We may assume $k\in \{i,j\}$ else the derivative is $0$ and by symmetry need only check $\frac{\partial \x}{\partial q_i}$. Differentiating \eqref{eq:xi2}, $$\frac{\partial \x}{\partial q_i} = \frac{\al (q_j-\x) + q_j} {1+\al(\s-\x+q_j-\x)},$$ clearly finite for $\al>0$ since recalling \eqref{eq:mu}, $\s-\x$ and $q_j-\x$ are elements of the pseudo-marginal and hence are non-negative (or use Lemma~\ref{lem:xiub}).
\end{proof}
%

\subsection{Submodularity}

\begin{theorem}
If a binary pairwise MRF is submodular on an edge $(i,j)$, i.e. $\al>0$, then the multi-label discretized MRF for any discretization $\mathcal{D}$ is submodular for that edge. In particular, if the MRF is fully associative/submodular, i.e. $\al>0 \; \forall (i,j) \in \cal{E}$, then the multi-label discretized MRF is fully submodular for any discretization.
\end{theorem}
\begin{proof}
For any edge $(i,j)$, let $f$ be the pairwise function $f_{ij}$ from \eqref{eq:f} and note the submodularity requirement \eqref{eq:sub}. 
Let $x=(x_1,x_2)$, $y=(y_1,y_2)$ be any points in $[0,1]^2$. Define $s(x,y)=(s_1,s_2)=(\min(x_1,y_1),\min(x_2,y_2))$, and $t(x,y)=(t_1,t_2)=(\max(x_1,y_1),\max(x_2,y_2))$. Let $g(x,y)=f(s_1,s_2)+f(t_1,t_2) - f(s_1,t_2)-f(s_2,t_1)$, call this the submodularity of the rectangle defined by $x,y$. We must show $g(x,y)\leq 0$. Note $f$ is continuous in $[0,1]^2$ hence so also is $g$. We shall show that $\forall x,y \in (0,1)^2, \, g(x,y)<0$ then the result follows by continuity.

Assume $x,y \in (0,1)^2$. Consider derivatives of $f$ in the compact set $R=[s_1,t_1] \times [s_2,t_2]$. Using \eqref{eq:1stderiv} and Lemma \ref{lem:mugz}, first derivatives exist and are bounded. By Theorem \ref{thm:2nderiv} and Lemma \ref{lem:bound3} the same holds for second and third derivatives. Further, Theorem \ref{thm:2nderiv} and Lemma \ref{lem:xiblb} show that $\frac{\partial^2f}{\partial q_i \partial q_j} = \frac{\partial^2f}{\partial q_j \partial q_i} < 0$.

If a rectangle is sliced fully along each dimension so as to be subdivided into sub-rectangles then summing the submodularities of all the sub-rectangles, internal terms cancel and we obtain the submodularity of the original rectangle.

Hence there exists an $\epsilon$ such that if we subdivide the rectangle defined  by $x,y$ into sufficiently small sub-rectangles with sides $< \epsilon$ and apply Taylor's theorem up to second order with the remainder expressed in terms of the third derivative evaluated in the interval, then the second order terms dominate and the submodularity of each small sub-rectangle $<0$. Summing over all sub-rectangles provides the result.
\end{proof}

\subsection{Second derivatives for singleton terms}

Let $f_i(q_i)$ be the singleton terms from \eqref{eq:F} for  $X_i$. The only non-zero derivatives are with respect to $q_i$.
\begin{align*}
f_i(q_i) &=  -\h q_i + (z_i-1) S_i(q_i) \\
\frac{\partial f_i}{\partial q_i} &= -\h - (z_i-1) [ \log q_i - \log (1-q_i)] \\
\frac{\partial^2 f_i}{\partial q_i^2} &= -(z_i-1) \frac{1}{q_i(1-q_i)} \leq 0 \text{ for a connected graph.}
\end{align*}
\begin{equation}\label{eq:2dsingle}
\text{Hence,} -\frac{z_i-1} {\eta_i(1-\eta_i)} \leq \frac{\partial^2 f}{\partial q_i^2} \leq 0 , \; \eta_i = \min(A_i,B_i).
\end{equation}

\section{Approximating the Global Optimum for an Associative Model}\label{sec:final}

We now assemble earlier results to form the complete matrix $H$ of second derivatives of the Bethe free energy $F$ and use this to bound the error between the discretized optimum and the global Bethe optimum. In this section we assume the model is associative.  Define the \textit{Bethe box} to be the orthotope (sometimes called a hyper-cuboid) given by $q_i \in [A_i,1-B_i] \; \forall i \in \cal{V}$.

At the optimum (or any stationary point), all first derivatives are zero. If we choose our discretization mesh $\cal{D}$ to be sufficiently fine then we can be sure that a point in the mesh is within distance $\delta$ of a true optimum. In particular, if we choose each $D_i$ so that in the $q_i$ dimension every point in $[A_i,1-B_i]$ is within distance $\gamma$, then $\delta^2 \leq n \gamma^2$. 

Using a first order Taylor expansion of $F$ around a true optimum, with the remainder expressed in terms of the second derivative, the error of our discretized optimum versus the true Bethe optimum  $\leq \frac{1}{2}\Lambda \delta^2$, where $\Lambda$ is the largest eigenvalue of $H$ evaluated at some intermediate point, which we shall bound.  Observe that the Bethe optimum (any stationary point) must lie within the Bethe box, and hence we may assume also that all mesh points are inside since it would be pointless to check outside it. We shall bound the largest eigenvalue of $H$ anywhere within the Bethe box.\footnote{This value can also be used to find an approximately stationary point \cite{Shin12} if required by considering the Taylor expansion of $F'$ around a stationary point.} 

Note that our error is one-sided since our discretized optimum can never be better than the true optimum. This may facilitate further analysis to find a better approximation by using points in the neighborhood to estimate the likely error. 

\subsection{Complete matrix of second derivatives}

 Theorem \ref{thm:2nderiv} and \eqref{eq:2dsingle} provide all the terms. 

\begin{lemma}\label{lem:diag}
All entries on the main diagonal of $H$ are strictly positive, all others are $\leq 0$.
\end{lemma}
\begin{proof}
Apply Theorem \ref{thm:2nderiv}. If $(i,j) \in \cal{E}$ then $H_{ij} = (q_iq_j-\x) / T_{ij} \leq 0$. If $(i,j) \notin \cal{E}$, $i \neq j$ then $H_{ij}=0$. On the main diagonal
\begin{align}\label{eq:Hii}
H_{ii} &= - \frac{z_i-1}{q_i (1-q_i)} + \sum_{j \in \N(i)} \frac{q_j(1-q_j)}{T_{ij}} \\
&\geq \frac{1-z_i}{q_i (1-q_i)} + \!\!\!\! \sum_{j \in \N(i)} \! \frac{q_j(1-q_j)}{q_iq_j(1-q_i)(1-q_j)} = \frac{1}{q_i(1-q_i)}.
\nonumber
\end{align}
\end{proof}
\subsection{Max eigenvalue \& complexity bound}

We have shown that $H$ is a real symmetric matrix with strictly positive main diagonal and all other entries $\leq 0$. To further bound the entries we derive a lower bound for $T_{ij}$ at any point in the Bethe box. Define $K_{ij}=\eta_i \eta_j (1-\eta_i)(1-\eta_j) \frac{2\al+1}{(\al+1)^2}$. All terms are known from the data prior to the discrete optimization.

\begin{lemma}\label{lem:Tlb}
At any point in the Bethe box, $T_{ij} \geq K_{ij}$.
\end{lemma}
\begin{proof}
Using Theorem \ref{thm:2nderiv} and Lemma \ref{lem:xiub}, 
\begin{align*}
T_{ij} &\geq q_i q_j (1-q_i)(1-q_j) - \Big(\frac{\al m (1-M)}{1+\al}\Big)^2 \\
 &\geq  q_i q_j (1-q_i)(1-q_j) \Big[1-\Big(\frac{\al}{1+\al} \Big)^2 \Big].
\end{align*} \end{proof}

\begin{theorem}\label{thm:ab}
At any point in the Bethe box, each entry $H_{ij}$ satisfies $-a \leq H_{ij} \leq b$ where 
\begin{align*}
a &=\frac{1}{4} \max_{(i,j)\in \cal{E}} \frac{\al}{\al+1} \frac{1}{K_{ij}} \\
 &= \max_{(i,j) \in \cal{E}} \frac{\al(\al+1)}{4(2\al+1) \eta_i \eta_j (1-\eta_i)(1-\eta_j)}, \nonumber\\
 b &= \max_{i \in \mathcal{V}} \frac{1}{\eta_i(1-\eta_i)} \Big( 1-z_i + \sum_{j \in \N(i)} \frac{(\al+1)^2}{2\al+1} \Big).
\end{align*}
\end{theorem}
\begin{proof}
For any edge $(i,j) \in \cal{E}$, 
$$-H_{ij} = \frac{\x-\s q_j}{T_{ij}} \leq \frac{ m (1-M)\al}{1+\al} \frac{1}{K_{ij}} \leq \frac{1}{4} \frac{\al}{1+\al}\frac{1}{K_{ij}}.$$
Using \eqref{eq:Hii} and the expression from the proof of Lemma \ref{lem:Tlb},
\begin{align*}
H_{ii} &\leq \frac{1-z_i}{\eta_i(1-\eta_i)} + \sum_{j \in \N(i)} \frac{1}{ q_i (1-q_i)\Big[1-\Big(\frac{\al}{1+\al} \Big)^2 \Big] } \\
&\leq \frac{1}{\eta_i(1-\eta_i)} \Big( 1-z_i + \sum_{j \in \N(i)} \frac{(\al+1)^2}{2\al+1} \Big).
\end{align*}\end{proof}

Since $\al+1 < 2 \al +1$ we have the corollary that $H_{ii} < \frac{1 + \sum_{j \in \N(i)} \al} {\eta_i (1-\eta_i)}$. We remark that at any minimum of the Bethe free energy, all eigenvalues are $\geq 0$ so at these locations the maximum eigenvalue $\leq$ Tr $H < \sum_{i \in \mathcal{V}} \frac{1}{\eta_i (1-\eta_i)} + \sum_{(i,j) \in \mathcal{E}} \al \Big( \frac{1}{\eta_i (1-\eta_i)} + \frac{1}{\eta_j (1-\eta_j)} \Big)$. 

In order to bound the largest eigenvalue, we may use recent results such as Corollary 2 in \cite{Zhan05}, although we suspect that the particular properties of $H$ given in Lemma \ref{lem:diag} may admit more precise bounds.

Here we use the following elementary bound which allows us to relate to the concepts of sparsity or maximum degree as in \cite{Shin12}. Let $\Sigma$ be the proportion of non-zero entries in $H$ so the number of non-zero entries is $n^2 \Sigma \leq n + n\Delta \Rightarrow \Sigma \leq \frac{\Delta+1}{n}$, since we have the main diagonal terms and two entries for each edge. Let $\Omega=\max(a,b)$ from Theorem \ref{thm:ab}, we have
\begin{equation}\label{eq:lambda}
\Lambda \leq \sqrt{\text{tr}(H^T H)} \leq \sqrt{\Sigma n^2 \Omega^2} = n \Omega \sqrt{\Sigma}.
\end{equation}

Returning to the reasoning at the start of this section \ref{sec:final}, note that by using $N_i$ points in $D_i$ we can ensure $\gamma \leq (1-B_i-A_i)/(N_i+1)$.  Using worst case Bethe bounds ($A_i=B_i=0$) we achieve maximum $\gamma$ distance in each dimension with $\frac{1}{\gamma}$ points for each variable, so the total number of nodes in the max-flow graph we need to solve the multi-label graph cuts problem is $N \leq \frac{n}{\gamma}$.  We require $n \gamma ^2 \leq \frac{2 \epsilon}{\Lambda}$ hence $N^2 \geq \frac{n^3 \Lambda}{2 \epsilon}$. Using \eqref{eq:lambda} it is sufficient if $N^2 \geq \frac{n^4 \Omega \sqrt{\Sigma}}{2 \epsilon}$. Graph cuts is a max-flow algorithm for which there are push-relabel methods guaranteed to run in time $O(N^3)$ \cite{Gol88}. Hence our algorithm has worst case run time of $\frac{n^6 \Omega^{3/2} \Sigma^{3/4}}{\sqrt{8} \epsilon^{3/2}} = O(n^6 \Sigma^{3/4} \Omega ^{3/2} \epsilon^{-3/2})$. However, in practice runtime for this class of problem using the Boykov-Kolmogorov algorithm \cite{BoyKol04} often approaches $O(N)$ for dramatically improved performance.

Note $\Omega$ above may depend of $n$. For our analysis in this paper we assumed the reparameterization in \eqref{eq:E} but a natural specification avoiding bias is to provide maximum possible values $W^*$ and $\theta^*$ with
\begin{align*}
\theta_{ij} &= \begin{pmatrix} W_{ij}/2 & 0 \\ 0 & W_{ij}/2 \end{pmatrix} \text{ s.t. } W_{ij} \leq W^* \; \forall (i,j) \in \cal{E} \\
|\h| &\leq \theta^* \; \forall i \in \cal{V}.
\end{align*}
The required reparameterization for edge $(i,j)$ takes $\h \leftarrow \h - W_{ij}/2$, hence reparameterizing all edges takes $\h \leftarrow \h - \sum_{j \in \N(i)} W_{ij}/2$. A sufficient condition for $\frac {1}{\eta_i (1-\eta_i)}$ to have a polynomial upper bound is that the maximum degree $\Delta := \max_{i \in \cal{V}} z_i = O(\log n)$, the same degree restriction as in \cite{Shin12}. In this case, $\frac{1}{\eta_i (1-\eta_i)} = O(\exp (\theta^* + \Delta W^* /2))$.

Regarding Theorem \ref{thm:ab}, now $a= O(\exp(W^*+2\theta^*+\Delta W^*)), b= O(\Delta \exp(W^*+\theta^*+\Delta W^*/2)$
with $\Omega=\max(a,b)$ and $\Sigma=O(\Delta/n)$ yielding the polynomial result. 

\section{Conclusion \& Extensions}

To our knowledge, we have proved the first PTAS for the global optimum of the Bethe free energy of an associative binary pairwise MRF. In doing so we derived a range of other results, including several for general edges and models (associative or not), which may prove useful in their own right, including Bethe bound propagation. 

Although our algorithm is only weakly polynomial, we are not sure if it is possible to do better. Note that if we make no restriction on input parameters, then potentially $\alpha$ values could be infinite, corresponding to probability distributions with exactly zero probability for some states (which may be reasonable), and this will lead to infinite derivatives as some pseudo-marginal entries will be driven to 0. 

\cite{Tar11} has shown that graph cuts is in a strong sense equivalent to max-product belief propagation with careful scheduling and damping. Together with our result this shows an interesting link between max-product and sum-product techniques. One direction to explore is how sum-product belief propagation fares using a scheme similar to \cite{Tar11}. 

We note that our approach immediately also applies to approximating optimum mean field marginals.  In addition, it may readily extend to allow approximate marginal inference for multi-label and third order submodular MRFs, both of which can be mapped to equivalent associative binary pairwise MRFs \cite{SchFla06,RKAT08}.

%
%
%
%


\bibliographystyle{plain}
\bibliography{references}

\end{document}